\documentclass{article}
\usepackage{proceed2e}

\newcommand{\shrink}[1]{}
\newcommand{\ignore}[1]{\relax}

\newcommand{\e}{\epsilon}
\newcommand{\E}{\mathbf{E}}

\newcommand{\D}{Q} 
\newcommand{\pecoc}{\operatorname{pecoc}} 
\newcommand{\obj}{\operatorname{obj}} 
\newcommand{\path}{T} 

\newcommand{\set}[1]{\left\{ {#1} \right\}}
\newcommand{\abs}[1]{\left| {#1} \right|}
\newcommand{\parens}[1]{\left( {#1} \right)}
\newcommand{\cond}{\mid}
\newcommand{\rbar}{\bar{r}}
\newcommand{\y}{\upsilon}
\def\isleft(#1,#2){\operatorname{right}_{#1}(#2)}

\usepackage[ruled,vlined]{algorithm2e}
\usepackage{algorithmic}
\usepackage{latexsym,amsmath,amsthm,amssymb,graphicx}

\newtheorem{theorem}{Theorem}
\newtheorem{lemma}[theorem]{Lemma}
\newtheorem{claim}[theorem]{Claim}

\newtheorem{definition}[theorem]{Definition}

\title{Conditional Probability Tree Estimation Analysis and
  Algorithms}

\author{Alina Beygelzimer \\
IBM Research\\
beygel@us.ibm.com \\
\And John Langford\\
Yahoo! Research\\
jl@yahoo-inc.com\\
\And Yuri Lifshits\\
Yahoo! Research\\
yury@yury.name\\
\And Gregory Sorkin\\
IBM Research\\
sorkin@us.ibm.com\\
\And Alex Strehl\\
Yahoo! Research\\
astrehl@gmail.com
}

\begin{document}

\maketitle

\begin{abstract}
We consider the problem of estimating the conditional probability of a
label in time $O(\log n)$, where $n$ is the number of possible labels.
We analyze a natural reduction of this problem to a set of binary
regression problems organized in a tree structure, proving a regret
bound that scales with the depth of the tree.  Motivated by this
analysis, we propose the first online algorithm which provably
constructs a logarithmic depth tree on the set of labels to solve this
problem.  We test the algorithm empirically, showing that it works
succesfully on a dataset with roughly $10^6$ labels.
\end{abstract}

\section{Introduction}
The central question in this paper is how to efficiently estimate the
conditional probability of label $y\in \{1,\ldots,n\}$ given an
observation $x\in X$.  Virtually all approaches for solving this
problem require $\Omega(n)$ time. A commonly used one-against-all approach,
which tries to predict the probability of label $i$ versus all other
labels, for each $i\in\{1,\ldots,n\}$, 
requires $\Omega(n)$ time per training example.
Another common $\Omega(n)$ approach is to
learn a scoring function $f(y,x)$ and convert it into a
conditional probability estimate according to $f(y,x)/Z(x)$, where
$Z(x) = \sum_i f(i,x)$ is a normalization factor.

The motivation for dealing with the computational difficulty is the
usual one---we want the capability to solve otherwise unsolvable
problems.  For example, one of our experiments involves a
probabilistic prediction problem with roughly $10^6$ labels and $10^7$
examples, where any $\Omega(n)$ solution is intractable.

\subsection{Main Results}
In Section~\ref{sec:online}, we provide the first online
supervised learning algorithm that trains and
predicts with $O(\log n)$ computation per example.
The algorithm does not require 
knowledge of $n$ in advance; it adapts naturally as new labels are
encountered.

The prediction algorithm uses a binary tree where regressors
are used at each node to predict the conditional probability that the
true label is to the left or right.  The probability of a leaf is
estimated as the product of the appropriate conditional probability
estimates on the path from root to leaf.  In our experiments, we
use linear regressors trained via stochastic
gradient descent.

The difficult part of this algorithm is constructing the tree itself.
When the number of labels is large, it becomes critical to construct
easily solvable binary problems at the nodes.  In
Section~\ref{sec:online-analysis}, we introduce a tree-construction
rule with two desirable properties.  First, it always results in depth
$O(\log n)$.  It also encourages natural problems by minimizing
expected loss at the nodes.  The technique used in the algorithm is
also useful for other prediction problems such as multiclass
classification.

We test the algorithm empirically on two datasets 
(in Section~\ref{sec:online-experiments}), and find that it both
improves performance over naive tree-building approaches and competes
in prediction performance with the common one-against-all approach,
which is exponentially slower.

Finally, we analyze a broader set of logarithmic time probability
estimation methods.  In Section~\ref{sec:cpt} we prove that any tree
based approach has squared loss bounded by the tree depth squared
times the average squared loss of the node regressors used.  In
contrast, the PECOC approach~\cite{pecoc} has squared loss bounded by
just $4$ times the average squared loss but uses $\Omega(n)$
computation.  This suggests a tradeoff between computation and squared
loss multiplier.  Section~\ref{sec:hpecoc} describes a $k$-parameterized 
construction achieving a ratio of 
$4 (\log_k n)^2 \left( \frac{k-1}{k} \right)^2$ 
while using $O(k \log_k n)$ computation, where $k=2$ gives 
the tree approach and $k=n$ gives PECOC.

\subsection{Prior Work}
There are many methods used to solve conditional probability
estimation problems, but very few of them achieve a logarithmic
dependence on $n$.  The ones we know are batch constructed regression
trees, C4.5~\cite{C45}, ID3~\cite{ID3}, or Treenet~\cite{Treenet},
which are both too slow to consider on datasets with the scale of
interest, and incapable of reasonably dealing with new labels
appearing over time.

Mnih and Hinton~\cite{MH} constructed a special purpose tree-based
algorithm for language modeling, which is perhaps the most similar
previous work.  The algorithm there is 
specialized to word prediction and is substantially slower since it
involves many iterations through the training data.  However, the
general analysis we provide in Section~\ref{sec:cpt} applies to their
algorithm.  We regard the empirical success of their algorithm as
further evidence that tree-based approaches merit investigation.
\subsection{Outline}

Section~\ref{sec:static} states and analyses
methods for logarithmic time probabilistic prediction {given} a
tree structure.  Section~\ref{sec:online} gives an
algorithm for building the tree structure.
The analysis in the first section is sufficiently general so that it
applies to the second.

\section{Problem Setting} \label{sec:setting}
Given samples from a distribution $P$ over $X\times Y$, where $X$
is an arbitrary observation space and $Y = \{1,\ldots, n\}$,
the goal is to estimate the conditional probability $P(y\mid x)$
of a label $y\in Y$ for a new observation $x\in X$.

For an estimator $\D(y\mid x)$ of $P(y\mid x)$,
the \emph{squared loss} of $\D$ with respect to $P$ is defined
as
\begin{align}
\ell_P(\D) &= \E_{(x,y)\sim P} (P(y \mid x) - \D(y \mid x))^2.
 \label{lossdef}
\end{align}
It is more common to define an observable squared loss where $P(y\, |\, x)$ 
in equation~(\ref{lossdef}) is replaced by $1$.  
We consider
\emph{regret} with respect to the common definition, since it is well
known that the difference between observable squared loss and the
minimum possible observable squared loss is equal to $\ell_P(\D)$.  We
therefore use regret and squared loss interchangeably in this paper.

It is well known that squared loss is a strictly proper scoring
rule~\cite{Brier}, thus $\ell_P(Q)$ is uniquely minimized by $Q=P$.
Our analysis focuses on squared loss because it is a bounded proper
scoring rule.  The boundedness implies that convergence guarantees
hold under weaker assumptions than for unbounded proper scoring rules
such as log loss.

\section{Probabilistic Prediction Given a Tree}
\label{sec:static}
This section assumes that a tree structure is given, and analyzes
how to use it for probabilistic logarithmic time prediction.

\subsection{Conditional Probability Tree} \label{sec:cpt}
Consider a fixed binary tree whose leaves are the $n$
labels.  For a leaf node $y\in Y$, let $\path(y)$ be the set of
non-leaf nodes on the path from the root to $y$ in the tree.

Each non-leaf node $i$ is associated with the regression problem of
predicting the probability, under $P$, that the label $y$ of a given
observation $x\in X$ is in the left subtree of $i$, conditioned on
$i\in \path(y)$.  The following procedure shows how to transform
multiclass examples into binary examples for each non-leaf node in the
tree.  Here $\isleft(i,y)$ is $0$ when $y$ is in the left subtree
of node $i$, and $1$ otherwise.

\SetAlFnt{\normalsize}
\begin{algorithm}
\dontprintsemicolon
\caption{Conditional Probability Tree Training (training set $S$, regression algorithm $R$)}
\label{rt-train}
\ForEach{internal node $i$}{
$S_i \leftarrow \emptyset$
}
\ForEach{example $(x,y)\in S$}{
\ForEach{node $i\in \path(y)$}{
Add $(x,\isleft(i,y))$ to $S_i$.
}}
\ForEach{internal node $i$}{
{\bf train} {$f_i = R(S_i)$}
}
\end{algorithm}
\noindent
Given a new observation $x\in X$ and a
label $y\in Y$, we use the learned
binary regressors $f_i$ to estimate $P(y\mid x)$.
Letting $\D_i(1 \mid x) = f_i(x)$ and
$\D_i(0\mid x) = 1-f_i(x)$, we define the estimate
\begin{align}
\D(y\mid x) &= \prod_{i\in \path(y)} \D_i(\isleft(i,y)\mid x) .  \label{Qdef}
\end{align}

\noindent
\subsubsection{Analysis of the Conditional Probability Tree}
\label{sec:binary}

Algorithm~\ref{rt-train} implicitly defines a distribution $P_i$
over $X\times \{0,1\}$ induced at node $i$:
A sample from $P_i$ is obtained by drawing $(x,y)$ according to $P$
until $i\in\path(y)$, and outputting $(x,\isleft(i,y))$
(although we never explicitly perform this sampling).
The following theorem bounds the squared loss of $\D$ given
the average squared loss of the binary regressors.

\begin{theorem} \label{tree:cor} 
For any distribution $P$, any set of node estimators $Q_i$, and any
pair $(x,y)$, with $Q$ given by equation~\eqref{Qdef},
\begin{align*}
(Q& (y \mid x) -P(y\mid x))^2 \\
   & \leq d^2 \; \E_i \parens{ \D_i(\isleft(i,y) \mid x)-P_i(\isleft(i,y)\mid x) }^2 ,
\end{align*}
where $d = |\path(y)|$ and
the expectation is over $i$ chosen uniformly at random from $\path(y)$.
\end{theorem}

\begin{proof}
We use Lemma~\ref{tree:thm}.  Using the notation of its proof,
observe that \begin{align*} \parens{ \sum_{i=1}^d \abs{q_i - p_i} }^2 &=
 d^2 \parens{\E_i \abs{q_i - p_i}}^2 \\ &\leq
d^2 \E_i \parens{|q_i-p_i|^2}\end{align*} using Jensen's
inequality.
\end{proof}

Most of the theorem is proved with the following core lemma.
For a node $i$ on the path from the root to label $y$,
define $p_i=P_i(\isleft(i,y)\mid x)$, the
conditional probability that the label is consistent with
the next step from $i$ given
that all previous steps are consistent.
Similarly define $q_i = \D_i(\isleft(i,y) \mid x)$.

\bigskip
\begin{lemma}\label{tree:thm}
For any distribution $P$,
any set of node estimators $Q_i$,
and any pair $(x,y)$,
with $Q$ given by equation~\eqref{Qdef},
\begin{align*}
\abs{ Q(y \mid x)-P(y\mid x)}
& \leq \sum_{i\in\path(y)} | q_i - p_i | \prod_{j\neq i} \max \{ p_j,q_j\} \\
   & \leq \sum_{i \in \path(y)} \abs{ q_i-p_i}. \\
\end{align*}
\end{lemma}

\vskip -.2in
The last inequality is the simplest---it says the differences in
errors add.  However, the quantity after the first inequality can be
much tighter. 

\begin{proof}
We first note that
$$ \abs{Q(y \mid x) - P(y \mid x)} \leq \prod_i \max\{p_i,q_i\} - \prod_i \min \{ p_i,q_i\}$$
since $\prod_i \max\{p_i,q_i\} \geq \max \{Q(y \mid x), P(y \mid x) \}$ and
$\prod_i \min\{p_i,q_i\} \leq \min \{Q(y \mid x), P(y \mid x) \}$.

We use a geometric argument.  With
$\prod_i \min\{p_i,q_i\}$ defining the volume of one ``corner'' of a
cube with sides $\max\{p_i,q_i\}$, slabs $|q_i-p_i| \prod_{j\neq i} \max \{ p_j,q_j\}$
  fill in the remaining volume (with overlap).  Consequently, we can
  bound the difference in volume as

\begin{align*}
\prod_i \max\{p_i,q_i\} &- \prod_i \min \{ p_i,q_i\}\\
& \leq \sum_i | q_i - p_i | \prod_{j\neq i} \max \{ p_j,q_j\} \\
& \leq \sum_i | q_i - p_i |,
\end{align*}
since all $p_j$ and $q_j$ are bounded by 1.
\end{proof}

As suggested by the proof, the lemma's bound can be asymptotically
tight.  If all $p_i$ are equal to some $p$ and all $|q_i-p_i|$ are
small, the left side is approximately $p^{d-1} \sum_i |q_i-p_i| = d
p^d \E |q_i-p_i|$, a factor $p^d$ times the right side.

\subsection{Conditional PECOC} \label{sec:hpecoc}
The conditional probability tree is as computationally tractable as we
could hope for, but is not as robust as we could hope for.  For
example, the PECOC approach~\cite{pecoc} yields a squared loss
multiplier of $4$ independent of the number of labels.  Is there an
approach more robust than the tree, but requiring less computation
than PECOC?

We provide a construction which trades off between the extremes of
PECOC and the conditional probability tree.  The essential idea is to
shift from a binary tree to a $k$-way tree, where PECOC with $k-1$
regressors is used at each node in the tree to estimate the
probability of any child conditioned on reaching the node.  For
simplicity, we assume that $k$ is a power of 2, and $n$ is a power of
$k$.

\begin{theorem}
\label{thm:hp}
Pick a $k$-way tree on the set of $n$ labels, where $k$ is a power of 2.
For all distributions $P$ and all sets of learned regressors,
with $k-1$ regressors per node of the tree, for all pairs $(x,y)$,
$$
 (\D(y \mid x)-P(y\mid x) )^2 \leq
    4(\log_kn)^2\left(\frac{k-1}{k}\right)^2\epsilon^2,
$$
where $\epsilon^2$ is the average squared loss of the $(k-1)\log_kn$
questioned regressors.
\end{theorem}

\begin{proof}
The proof is by composition of two lemmas.

In each node of the tree, Lemma~\ref{pecoc:thm} bounds the power of
the adversary to disturb the probability estimate as a function of the
adversary's regret.  Similarly, Lemma~\ref{tree:thm} bounds the power
of the adversary to induce an overall misestimate as a function of the
adversary's power to disturb the estimates within each node on the path.
\end{proof}

The curve below illustrates how the construction trades off
computation for a better regret bound as a function of $k$.

\includegraphics[angle=270,width=.48\textwidth]{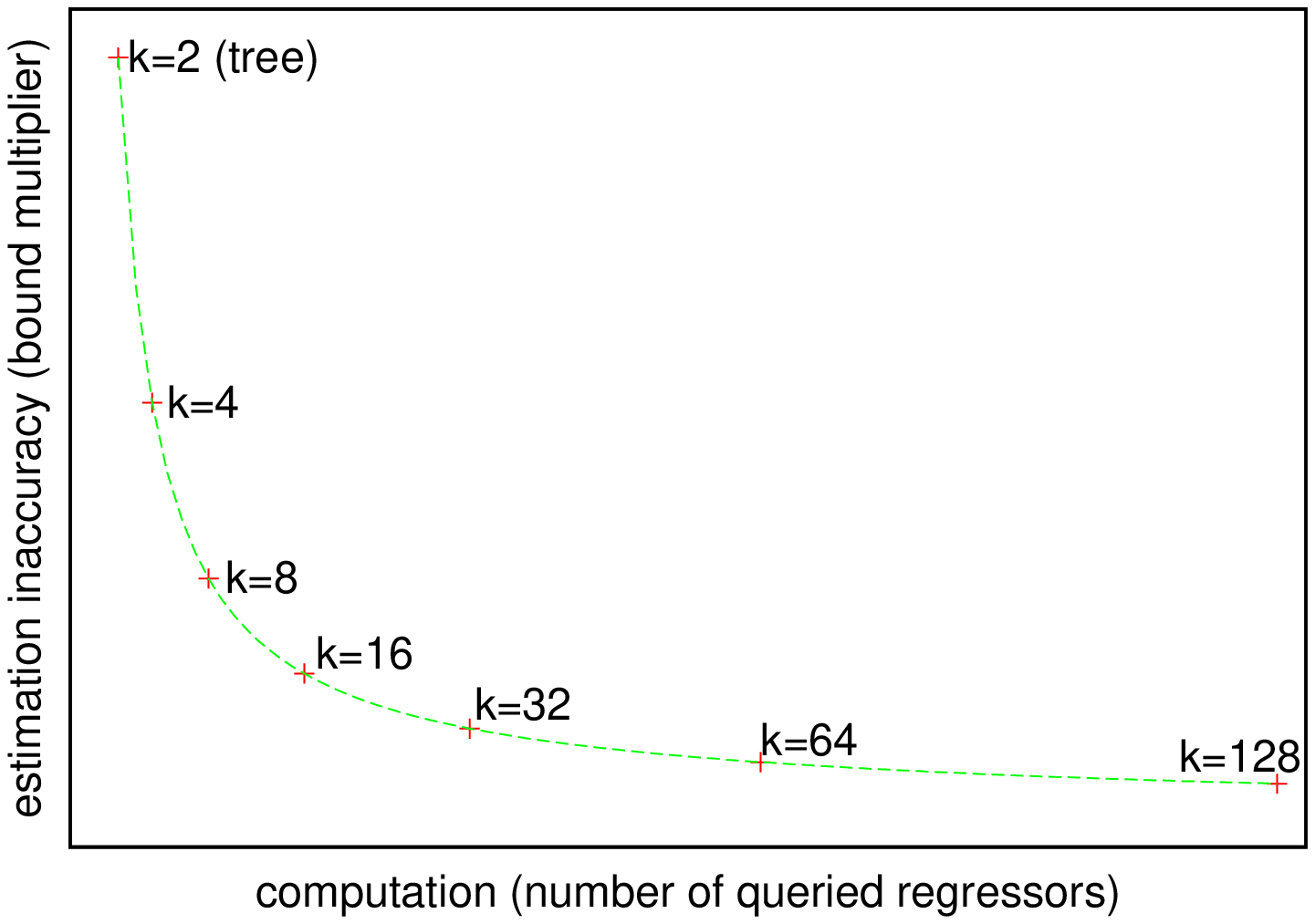}

To complete the proof of Theorem~\ref{thm:hp} we describe the PECOC
construction in Section~\ref{sec:pe} and prove Lemma~\ref{pecoc:thm} in
Section~\ref{sec:careful}.

\subsubsection{The PECOC Construction}
\label{sec:pe}
The PECOC construction is defined by a binary matrix $C$ with each
column a label and each row defining a regression problem.  The
regression problem corresponding to row $i$ is to predict the
probability given $x$ that the correct label is in the subset
\begin{align}
Y_i = \set{ y\in Y \colon C(i,y)=1}. 	\label{subsets}
\end{align}

We use an explicit family of Hadamard codes
given by the recursive formula
$$
C_2 =
\begin{bmatrix}
  1 & 1      \\
  1 & 0
\end{bmatrix},
\quad
C_{2^t} =
\begin{bmatrix}
  C_t & C_t      \\
  C_t &  1-C_t
\end{bmatrix} .
$$ We use a matrix $C_{2^t}$ with $2^{t-1} \leq n < 2^t$, noting that
its size $2^t$ is less than $2n$; if $2^t>n$ we simply add dummy
labels.  We henceforth assume without loss of generality that $n$ is a
power of~2.  We train PECOC according to the following algorithm.

\SetAlFnt{\normalsize}
\begin{algorithm}
\label{alg:PECOC}
\dontprintsemicolon
\caption{PECOC Training (training set $S$, regression algorithm $R$)}
\For{each row $i$ of $C$}{
Let $S_i = \{ (x,C(i,y)) : (x,y) \in S\}$\;
{\bf train} $r_i = R(S_i)$. \;
}
\end{algorithm}
Given a new observation $x\in X$ and a label $y\in Y$, PECOC uses the
binary regressors $r_i$ learned in Algorithm~\ref{alg:PECOC} to estimate
$P(y\mid x)$ using the formula
\begin{align}
\pecoc(y \mid x) = 		
 2\, \E_i \big[ & C(i,y)r_i(x) +
 \notag
  \\
  & (1-C(i,y))(1-r_i(x)) \big] - 1,
		\label{pecocest}
\end{align}
where the expectation is over $i$ drawn uniformly from the rows of
$C$.  The reason for this formula is clarified by the proof of
Lemma~\ref{pecoc:thm}.

\subsubsection{A Careful PECOC analysis}
\label{sec:careful}

The following theorem gives the precise regret bound, which follows
from the analysis in \cite{pecoc} but is tighter for small values of $n$
than the bound stated there.

\begin{lemma}\label{pecoc:thm}\emph{(PECOC regret~\cite{pecoc})
For all distributions $P$ and all sets of
regressors $r_i$ (as defined in Algorithm~\ref{alg:PECOC}), for all $x\in X$
and $y\in Y$,
\begin{align*}
(\pecoc(y \mid x) & - P(y\mid x))^2 \leq \\
 & 4\left(\frac{n-1}{n}\right)^2
\E_i (r_i-P(y\in Y_i\mid x))^2,
\end{align*}
where $Y_i$ 
is the subset defined by row $i$ per~\eqref{subsets}.
}
\end{lemma}

\begin{proof}
Since the code and the prediction algorithm
are symmetric with respect to set inclusion,
we can assume without loss of generality
that $y$ is in every subset (complementing all
subsets not containing $y$).
Thus every entry $C(i,y)=1$, and by \eqref{pecocest}
the PECOC output estimate of $P(y\mid x)$ is
$$ \pecoc(y \mid x) = \frac{2}{n} \sum_{i=1}^n r_i(x) - 1 . $$

\ignore{
If all subset predictions are perfect,
$$\sum_{i=1}^n r_i(x)= \frac{n}{2}(1+ P(y\mid x)),$$
using the fact that every label other
than $y$ appears in exactly half of the subsets.
Hence $\pecoc(y \mid x) = P(y\mid x)$ when the regressors are perfect.
}

\sloppypar
Let $\rbar_i(x) = P(y \in Y_i \cond x) = \sum_{\y \in Y_i} P(\y \cond x)$
denote the perfect subset estimators,
and write $r_i(x) = \rbar_i(x) + \e_i$.
By the nature of $C$, the label $y$ under consideration occurs in every
subset, and every other label $\y \neq y$ in exactly half the subsets, so that
\begin{align*}
 \sum r_i(x)
  &= \sum_i \left( \sum_{\y \in Y_i} P(\y \cond x) + \e_i \right)
  \\& = \sum_\y \sum_{i \colon Y_i \ni \y} P(\y \cond x) + \sum_i \e_i
  \\& = \sum_{\y \neq y} \frac n 2 P(\y \cond x) + n P(y \cond x) + \sum \e_i
  \\& = \frac n 2 (1+P(y \cond x)) + \sum \e_i .
\end{align*}
This gives $\pecoc(y \mid x) = P(y \cond x) + \frac2n \sum_i \e_i$,
for squared loss
$(\pecoc(y \mid x)-P(y \cond x))^2 = (\frac2n \sum_i \e_i)^2$.
One of the subsets, say the first, is trivial (it includes all labels),
and for it we stipulate the true probability $r_1=1$, so $\e_1=0$.
Letting $\E_i \e_i$ denote the mean of the other $n-1$ errors $\e_i$,
the squared loss is $(2 \frac{n-1}n \E_i \e_i)^2$,
establishing the theorem.
\end{proof}

\section{Online Tree Construction}
\renewcommand{\k}{\kappa}
\renewcommand{\a}{\alpha}
\newcommand{\ff}{H}
\label{sec:online}
The analysis of Section~\ref{sec:binary} applies to any binary tree,
and motivates the creation of trees which have small depth and small
regret at the nodes.  This leaves the question, ``Which tree should we
use?''  We give an online tree construction algorithm with several useful 
properties. In particular, 
the algorithm doesn't require any prior knowledge of the labels, and takes
$O(\log n)$ computation per example, when there
are $n$ labels.  The algorithm guarantees a tree with $O(\log n)$
maximum depth using a decision rule that trades off between depth and
ease of prediction.  

\subsection{Online Tree Building Algorithm}

Algorithm~\ref{alg:rt} builds and maintains a tree, whose leaves are
in one-to-one correspondence with the labels seen
so far.  Each node $i$ in the tree is associated with a regressor
$f_i:X\to [0,1]$.  Given a new sample $(x,y) \in X \times Y$, we
consider two cases.

If $y$ already exists as a label of some leaf in the tree,
then there is an associated root-to-leaf path and we can use the conditional
probability tree algorithms of the previous section to train and test on 
$(x,y)$, with one minor modification when training: we
add a regressor at the leaf and train it with the example $(x,0)$.

If $y$ does not exist in the tree, then the algorithm still
traverses the tree to some leaf $j$, 
using a decision rule 
that computes a direction (left or right) at each non-leaf node
encountered.  Once leaf $j$ is reached, it necessarily corresponds to
some label $y' \neq y$.  We convert $j$ to a non-leaf node with left
child $y'$ and right child $y$.  The regressor at node $j$ is
duplicated for $y'$.  A new regressor is created for $y$ and trained
on the example $(x,0)$.  

We now describe the decision rule used to decide
which way to go (left or right) at each 
non-leaf node $i$ encountered during the traversal.
First, let
$L_i$ denote the number of children to the left of node $i$, and $R_i$
the number to the right. If $f_i(x) > 1/2$, 
where $f_i(x)$ is the current prediction associated with node $i$ on $x$, then
the regressor favors the right subtree for this input, and
otherwise the left subtree.  If the regressor favors the side with the
smaller number of elements,
then this direction is chosen.  If the regressor favors the side with
more elements, then the algorithm
faces a dilemma.  On one hand, sending the new label to the right
would result in a more highly balanced tree, but on the other hand it
would result in a training sample disagreeing with the current
regressor's prediction.  Our resolution is to define an objective
function
$$
\obj(p,L,R,\alpha) 
 = (1-\alpha)2(p-\tfrac12) + \alpha \log_2{\tfrac{L}{R}}
$$
and send the label to the right of node $i$ if
\begin{equation}
\obj(f_i(x),L_i,R_i,\alpha) > 0. \label{rule}
\end{equation}
Here $\alpha$ is a free parameter set for the run of the entire
algorithm.  When $\alpha = 1$, the rule indicates that we should place
new labels on the side with fewer current labels, resulting in a
perfectly balanced tree.  When $\alpha = 0$, the direction chosen is
always the one currently favored by the regressor.  A trade-off between these
two objectives is provided by values of $\alpha$ between these two
extremes.

Pseudo-code is provided in Algorithm~\ref{alg:rt}.

\begin{algorithm}[!t]
\caption{Online conditional probability tree (CPT) Training 
 (regression algorithm $R$, aggressiveness $\alpha$)}
\label{alg:rt}
{\bf create} the root node $r$ \\
\ForEach{example $(x,y)$}
        {
            \If{$y$ has been seen previously}
               {For each $i \in \path(y)$, 
		    train $f_i$ with $(x,\isleft(i,y))$.}
               \Else{Set $i=r$.\\
                 \While{$i$ is not a leaf}{
                   {\bf if} $\obj(f_i(x),L_i,R_i,\alpha)>0$ {\bf then} $c = 1$ (right) \\
                   {\bf else} $c = 0$ (left) \\
                   Train $f_i$ with example $(x,c)$\\
	           Set $i$ to the child of $i$ corresponding to $c$
                 }
		 Create children of leaf $i$: \\
		   \hskip .3in left with a copy of $i$ (including $f_i$), \\
		   \hskip .3in right with label $y$ trained on $(x,0).$ \\

                 Train $f_i$ with $(x,1)$.\\
               }
        }
\end{algorithm}

\subsection{Online Tree Building Analysis}
\label{sec:online-analysis}
In this section we analyze Algorithm~\ref{alg:rt}.
Throughout the section, 
for any tree node under consideration, 
we will use $N$ for the total number of leaves 
under the node, $L$ the number on the left and $R$ on the right,
with $L+R=N$.
We note that rule \eqref{rule} is symmetric with respect to $L$ and~$R$.
%
%
We also define
$$\k = \frac1{1+2^{1-1/\a}} . $$
Claim~\eqref{bigside} will establish
that at most about a fraction $\k$ of the leaves can fall
on either side of a node,
with $\k=1/2$ for $\a=1$ and $\k \rightarrow 1$ as $\a \rightarrow 0$.

\begin{claim}
If a node has $L$ leaves in its left subtree, $R$ in the right,
and $N=L+R$ altogether,
if $R/N > \k$ then a new leaf is added to the left subtree
regardless of the prediction value $p$ at the node
(and symmetrically for~$L$).
\end{claim}
\begin{proof}
For any $p \in [0,1]$, 
\begin{align*}
\obj(p,L,R,\alpha) 
 & \leq
 (1-\alpha)2(1-\tfrac12) - (1-\alpha)
 \\ &= (1-\a) + \a \log_2 \tfrac L R,
\end{align*}
which is $<0$ (forcing a leaf to be added to the left)
if $L/R < 2^{\frac{\a-1}{\a}}$,
or equivalently if $R/N > \k$.
\end{proof}

\begin{claim} \label{bigside}
Under any non-leaf node, $L,R < \k N + (1-\k)$.
\end{claim}
\begin{proof}
We prove this inductively for $R$; 
the result for $L$ follows symmetrically.
A non-leaf node starts with one left and one right child,
and $R=L=1$, $N=2$ satisfies the claim.
Given that $R$, $L$, and $N$ satisfy the claim, 
we now prove that when a leaf is added,
so do the next values
$R'$ (either $R$ or $R+1$), 
$L'$ (respectively $L+1$ or $L$), and $N'=N+1$.
There are two cases.
If $R < \k N$ then 
$$R' \leq R+1 < \k N+1 = \k(N'-1)+1 = \k N'+1-\k . $$
If $R \geq \k N$ then the next addition is to $L$ not $R$,
and 
$$ R' = R \leq \k N+1-\k < \k N' +1-\k . $$
\end{proof}

\begin{theorem}
For all regressors at the nodes of the tree, for all learning problems
on $n$ labels, for all $\alpha \in (0,1]$ the depth of the tree is at most
${\log n}/{\log \k}+2$.
\end{theorem}
\begin{proof}
If the root node has $n$ leaves below it, 
then by the preceding claim a child (``depth 1'') of the root has
at most $\k n + (1-\k)$ leaves,
a grandchild has at most $\k^2 n + \k(1-\k)+(1-\k)$ leaves,
and a depth-$d$ child has at most 
$$ \k^d n+ \k^{d-1}(1-\k)+\cdots+k(1-\k)+(1-\k)
 \leq \k^d n + 1$$
leaves, using $\sum_{d=0}^\infty \k^d = 1/(1-\k)$.
With $d=-\lceil \ln n / \ln \k \rceil$, a depth-$d$ child
has at most 2 leaves, and thus further depth one, and
we add one more to account for the ceiling function.
\end{proof}

\begin{definition} A {\em disagreement} is the event when
a new label reaches a node, and the algorithm
decides to insert it in the subtree that is not preferred by the
regressor.
\end{definition}

That is, a disagreement occurs when the regressor's prediction is at most
$1/2$ and the label is inserted to the right, or when the
prediction is greater than $1/2$ and the label is inserted to the
left.

Note that the number of disagreements incurred when adding a new label (leaf)
is at most the depth of that leaf,
and as the tree evolves the ``same'' leaf 
(per the copying rule of the algorithm)
may become deeper but never shallower.
Thus the total number of disagreements incurred in building a tree
is at most the sum of the depths of all leaves of the final tree.

To get a grasp on this quantity,
for simplicity we disregard the 
additive $1-\k$ in Claim~\ref{bigside} coming from adding
vertices discretely, one at a time.
(The effect is most dramatic when a node has just two children,
$L=R=1$, and adding a leaf necessarily produces a lopsided tree
with $L=1$ and $R=2$ or vice-versa.
For large values of $L+R=N$ the effect of discretization is negligible.)

As usual, for a node in a tree let 
$L$ be the number of leaves in its left subtree, $R$ in the right, $N=L+R$.

\begin{theorem}
Let $T$ be an $n$-leaf binary tree in which for each node, 
$L,R \leq \k N$.
Then the total of the depths of the leaves of $T$ is at most
$d(n) = n \log n / \ff(\k)$,
where $\ff(\k) = -\k \log \k - (1-\k) \log (1-\k)$.
\end{theorem}
\begin{proof}
The proof is by induction on~$n$, 
starting from the base case $n=2$ where the total of the depths 
(or total depth for short) is~2.
It is well known that the entropy function $\ff(\k)$ is maximized by 
$\ff(1/2)=\log 2$,
so in the base case we do indeed have $2 \leq d(n)$ since
$d(n) \geq 2 \log 2 / \log 2 = 2$.

Proceeding inductively, the total depth for an $N$-leaf tree
with $L$- and $R$-leaf subtrees is the total depth of $L$ (at most $d(L)$),
plus the total depth of $R$ (at most $d(R)$),
plus $N$ (since each leaf is 1 deeper in the full tree).
Since $d(\cdot)$ is a convex function, the worst case comes from
the most unequal split, and applying the inductive hypothesis,
the total depth for $N$ is at most
\begin{align*}
N + & d(\k N) + d((1-\k)N)
 \\ & \leq
 N + \frac{\k N \log(\k N)} {\ff(\k)} + \frac{(1-\k)N \log((1-\k) N)}{\ff(\k)}
 \\ &=
 N + \frac{N}{\ff(\k)} (\k \log \k + \k \log N 
 \\ & \qquad \qquad
   + (1-\k) \log (1-\k) + (1-\k) \log N)
 \\ &=
 N + \frac{N}{\ff(\k)} (-\ff(\k) + \log N)
 \\ &=
 N \log N / \ff(\k)
 \\ &= d(N) ,
\end{align*}
completing the proof that $d(N)$ is an upper bound.
\end{proof}

\subsection{Experiments}
\label{sec:online-experiments}
We conducted experiments on two datasets.  The purpose of the first
experiment is to show that the conditional probability tree (CPT)
competes in prediction performance with existing exponentially slower
approaches.  To do this, we derive a label probability prediction
problem from the publicly available Reuters RCV1 dataset~\cite{rcv}.
The second experiment is a full-scale test of the system where an
exponentially slower approach is too intractable to seriously
consider.  We use a proprietary dataset that consists of webpages and
associated advertisements, where the derived problem is to predict the
probability that an ad would be displayed on the webpage.

Each dataset was split into a training and test set.  Each training
or test sample is of the form $(x,y)$.  The algorithms train on the
training set and produce a probabilistic rule $f(\cdot,\cdot)$ that
maps pairs of the form $(x,y)$ to numbers in the range $[0,1]$, where
we interpret $f(x,y)$ as an approximation to $P(y\mid x)$.  The algorithms
are evaluated on the test set by computing the empirical squared
loss,
$\sum_{(x,y)} (1 - f(x,y))^2$.
The algorithms are allowed to continue learning as they are tested,
however the predictions $f(x,y)$ used above are
computed before training on the sample $(x,y)$.  This type of
evaluation is called ``progressive validation'' \cite{PV} and
accurately measures the performance of an online algorithm.  In
particular, it is an unbiased estimate of the algorithm's performance
under the assumption that the $(x,y)$ pairs are identically and
independently distributed.  In the motivating applications of our
algorithm, we expect new labels to appear throughout the learning
process, which requires learning to occur continually in an online
fashion.  Thus, turning learning off and computing a ``test loss'' is
less natural.  Nevertheless, for the Reuters dataset, we verified that
the test loss and progressive validation are quite similar.  For the
web advertising dataset, the two measures were drastically different
(all methods performed much worse under test loss), due to the large
number of labels that appear only in the test set.

The CPT algorithm was executed with three tree-building construction
methods: a random tree where uniform random left/right decisions were
made until a leaf was encountered, a balanced tree according
to algorithm~\ref{alg:rt} with $\alpha=1$, and a general tree according
to algorithm~\ref{alg:rt} with $\alpha < 1$.  For the binary
regression problems (at the nodes), we used Vowpal Wabbit~\cite{VW},
which is a simple linear regressor trained by stochastic gradient
descent. 
One essential enabling feature of VW is
a hashing trick (described in~\cite{Aistat,Arxiv}) which allows us
to represent $1.7M$ linear regressors on a sparse feature space in a
reasonable amount of RAM.

\subsubsection{Reuters RCV1}
The Reuters dataset consists of about $800K$ documents, each assigned to
one or more categories.  A total of approximately 100 categories
appear in the data.  We split the data into a training set of $780K$
documents and a test set of $20K$ documents, opposite to its original
use.  
For each document $\mathrm{doc}$, we formed an example
of the form $(x,y)$, as follows.  The vector $x$ uses a
``bag of words'' representation of $\mathrm{doc}$, weighted by the
normalized TF-IDF scores, exactly as done in the paper \cite{rcv}.
The label $y$ is one of the categories assigned to $\mathrm{doc}$,
chosen uniformly at random if more than one category was assigned to
$\mathrm{doc}$.

We compared the CPT to the one-against-all algorithm, a standard
approach for reducing multi-class regression to binary regression. 
The one-against-all approach regresses on the probability of
each category $c$ versus all other categories.
Given a base
training example $(x,y)$, the example used
to train the regressor $f_c$ for category $c$ 
is $(x,I[y = c])$, where $I[\cdot]$ is the indicator
function.  Predictions for a new test example $(x,y)$ are
done according to $f_y(x)$.
The learning algorithm used
for training the binary regressors in both approaches
was incremental gradient descent with squared
loss.  For each algorithm, we ran several versions with different
learning rates, chosen from a coarse grid, and picked the setting that
yielded the smallest training error.  For the CPT algorithm,
we performed a similar search over $\alpha$.

The one-against-all approach used one pass over the
training data, while the CPT used two passes.  Note
that even with an additional pass, the CPT is much faster than
one-against-all for training, due to the fact that CPT requires
training only about $\log(\mathrm{number~of~categories}) = \log(103)$
regressors (nodes in the tree) per example, whereas one-against-all
trains one regressor per category.  On our machine, the CPT took 108
seconds to train, while one-against-all took 2300 seconds.  We use
Progressive Validation~\cite{PV} to compute an average squared loss
over the test set with results appearing in the following table, where
the confidence intervals are computed by Hoeffding's
inequality~\cite{Hoeffding} with $\delta = 0.05$.
\begin{center}
\begin{small}
\begin{tabular}{|c||c|}
\hline
One-against-all & $0.55 \pm .012$\\
\hline
CPT with a random tree & $0.56 \pm .012$\\
\hline
CPT with a balanced tree & $0.56 \pm .012$\\
\hline
CPT with an online tree ($\alpha = 0.6$) & $0.56 \pm .012$\\
\hline
\end{tabular}
\end{small}
\end{center}
The values are indeed mostly identical, but CPT achieved this performance with 
an order of magnitude less computation.

Note that in this problem, there is not much advantage in using our 
algorithm over using a random tree.  Since
there aren't many labels and there are many examples,
the structure of the tree is not very important. 
This is confirmed by 
running the algorithm with various different
random trees and observing little variability in squared
loss.

\subsubsection{Web Advertising}
We used a proprietary dataset consisting of about $50M$ pairs of
webpages and associated advertisments that were shown on the webpage.
There are about $5.8M$ unique webpages and $860K$ unique ads in the
dataset. The most frequent ad appeared in 
approximately 1.2\% of the cases.
The events were split into a training set of size $40M$, and
a test set of size $10M$ in time order.  Note that webpages and ads
both appear multiple times in the training and test sets.  For each
event, where an event consisted of a single ad being shown on a single
webpage, we create a sample $(x,y)$, where $x$ is a ``bag of words''
vector representation of the webpage, and $y$ is a unique ID
associated with the advertisement displayed.  The learning problem 
is predictinge $P(y\mid x)$, or 
the probability that the logging policy displays advertisement $y$
given webpage $x$.  Since $n$ is large, one-against-all would be
extremely slow.  The running time for our algorithm on
this dataset was about 60 minutes.  Multiplying by $860k/\log_2(860k)$
suggests a running time for one-aginst-all of about 5 years.

Besides the three versions of CPT described above, we tested one other
method we call the ``table-based'' method.  In the table-based method,
we simply predict $P(y\mid x)$ by the empirical frequency with which ad
$y$ was displayed on webpage $x$ in the training set.  The progressive
validation~\cite{PV} results of the four algorithms over the test set
appear in the following table with confidence intervals again computed
using Hoeffding's bound for $\delta = 0.05$.
\begin{center}
\begin{small}
\begin{tabular}{|c||c|c|}
\hline
Method & Squared Loss & Equivalent\\
\hline
Table& $0.812 \pm .00055$ & $10.11$\\
\hline
Random tree & $0.7742 \pm .00055$ & $8.32$\\
\hline
Balanced tree & $0.7725 \pm .00055$ & $8.25$\\
\hline
Online tree ($\alpha = 0.9$) & $0.7632 \pm .00055$ & $7.91$\\
\hline
Best possible & $0.665$ & $5.42$ \\
\hline
\end{tabular}
\end{small}
\end{center}
Here, the ``Equivalent'' column is the number of labels for which a uniform
random process produces the same loss.  The ``Best possible'' line
is an unachievable 
bound on performance found by examining the empirical frequency of 
ad-webpage pairs in the test set.

The magnitude of squared loss improvement is modest, but substantial
enough to be useful.  Since many of the webpages are seen many times,
the conditional distribution over ads can be approximated well by
empirical frequencies.  Thus, the table-based method forms a strong
baseline.  A small but significant fraction of the webpages were seen
only a few times, and for these webpages, it was necesssary to
generalize (predict which ads would appear based on which ads appeared
on pages similar to the current one).  On these examples, the tree
performed substantially better.

\end{document}